%% file: icml_submission.tex
\documentclass{article}

\usepackage{microtype}
\usepackage{graphicx}
\usepackage{subfigure}
\usepackage{booktabs} 
\usepackage{xspace}
\usepackage{amsmath}
\usepackage{placeins}
\usepackage{todonotes}

\usepackage{hyperref}

\input{mydef.tex}

\usepackage[accepted]{icml2023}
\icmltitlerunning{Energy Efficient Training of SNN using Local Zeroth Order Method}

\begin{document}

\twocolumn[
\icmltitle{Energy Efficient Training of SNN using Local Zeroth Order Method}



\icmlsetsymbol{equal}{*}

\begin{icmlauthorlist}
\icmlauthor{Bhaskar Mukhoty}{equal,mbzuai}
\icmlauthor{Velibor Bojkovi\'c}{equal,mbzuai}
\icmlauthor{William de Vazelhes}{mbzuai}
\icmlauthor{Giulia De Masi}{tii}
\icmlauthor{Huan Xiong}{mbzuai}
\icmlauthor{Bin Gu}{mbzuai}

\end{icmlauthorlist}

\icmlaffiliation{mbzuai}{Department of Macine Learning, Mohamed bin Zayed University of Artificial Intelligence, Abu Dhabi, UAE}
\icmlaffiliation{tii}{Technology Innovation Institute, Abu Dhabi, UAE}

\icmlcorrespondingauthor{Bin Gu}{jsgubin@gmail.com}
\icmlcorrespondingauthor{Huan Xiong}{huan.xiong.math@gmail.com}

\icmlkeywords{Spiking Neural Network, Zeroth Order, Surrogate Gradient}
\vskip 0.3in
]



\printAffiliationsAndNotice{\icmlEqualContribution} 

\begin{abstract}
Spiking neural networks are becoming increasingly popular for their low energy requirement in real-world tasks with accuracy comparable to the traditional ANNs. SNN training algorithms face the loss of gradient information and non-differentiability due to the Heaviside function in minimizing the model loss over model parameters. To circumvent the problem surrogate method uses a differentiable approximation of the Heaviside in the backward pass, while the forward pass uses the Heaviside as the spiking function. We propose to use the zeroth order technique at the neuron level to resolve this dichotomy and use it within the automatic differentiation tool. As a result, we establish a theoretical connection between the proposed local zeroth-order technique and the existing surrogate methods and vice-versa. The proposed method naturally lends itself to energy-efficient training of SNNs on GPUs. Experimental results with neuromorphic datasets show that such implementation requires less than $1\%$ neurons to be active in the backward pass, resulting in a 100x speed-up in the backward computation time. Our method offers better generalization compared to the state-of-the-art energy-efficient technique while maintaining similar efficiency.
\end{abstract}

\section{Introduction}

Biological neural networks are known to be significantly more energy efficient than their artificial avatars - the artificial neural networks (ANN). Unlike ANNs, biological neurons use spike trains to communicate and process information asynchronously. \cite{mainen1995reliability} To closely emulate biological neurons, spiking neural networks (SNN) use binary activation to send information to the neighboring neurons when the membrane potential exceeds the membrane threshold. The event-driven binary activation simplifies the accumulation of input potential and reduces the computation burden when the spikes are sparse. Specialized neuromorphic hardware \cite{davies2018loihi} is designed to carry out such event-driven and sparse computations in an energy-efficient way \cite{pfeiffer2018deep, kim2020spiking}.

There are broadly three categories of training SNNs: ANN-to-SNN conversion, unsupervised and supervised. The first one is based on the principle that parameters for SNN are inferred from the corresponding ANN architecture \cite{cao2015spiking,diehl2015fast,bu2021optimal}. Although training SNNs through this method achieves performance comparable to ANNs, it suffers from long latency needed in SNNs to emulate the corresponding ANN, or from energy expensive retraining of ANNs which is required in order to achieve near lossless conversion \cite{davidson2021comparison}. The unsupervised training is biologically inspired and uses local learning in order to adjust the parameters of the SNN \cite{diehl2015unsupervised}. Although it is the most energy efficient one among the three methods, as it is implementable on neuromorphic chips \cite{davies2018loihi}, it still lags behind in its performance compared to ANN-to-SNN conversion and supervised training.

Finally, supervised training is a method of direct training of SNNs by using back-propagation (through time). As such, it faces two main challenges. The first one is due to the nature of SNNs, or more precisely, due to the Heaviside activation of neurons (applied on the difference between the membrane potential and membrane threshold). As the derivative of the Heaviside function is zero, except at zero where it is not defined, back-propagation does not convey any information for the SNN to learn \cite{eshraghian2021training}. One of the most popular ways to circumvent this drawback is to use surrogate methods, where a derivative of a surrogate function is used in the backward pass during training. Due to their simplicity, surrogate methods have been widely used and have seen tremendous success in various supervised learning tasks \cite{shrestha2018slayer, neftci2019surrogate}. However, large and complex network architectures, the time-recursive nature of SNNs and the fact that the training is oblivious of the sparsity of spikes in SNNs, make surrogate methods quite a time and energy-consuming. 

In order to deal with energy (in)efficiency during direct training of SNNs, only a handful of methods have been proposed,  most of which are concerned with forward propagation in SNNs. \cite{alawad2017stochastic} uses stochastic neurons to increase energy efficiency during inference. More recently \cite{yan2022sparsereg} uses regularization during the training in order to increase the sparsity of spikes which reduces computational burden and energy consumption. \cite{cramer2022surrogate} performs the forward pass on a neuromorphic chip, while the backward pass still takes place on a standard GPU. However, these methods do not significantly reduce the computational weight of the backward pass. On the other side, \cite{nieves2021sparse} introduces a threshold for surrogate gradients (or, suggests using only a surrogate with bounded support). But, introducing gradient thresholds has a drawback of limiting the full potential of surrogates during training. 
In this paper, we propose a direct training method for SNNs which encompasses the full power of surrogate gradients in an energy efficient way. Based on zeroth order techniques and applied locally at neuronal level - hence dubbed Local Zeroth Order (\lzo) - our method is able to simulate arbitrary surrogate functions during the training, and at the same time significantly reduce the number of computational steps in the backward pass, which directly translates to energy saving.

The main contributions of the paper can be summarized as:
\begin{itemize}
\item We introduce zeroth order techniques in SNN training at a local level, while providing theoretical connections with surrogate gradient methods.
\item We experimentally demonstrate the main properties of \lzo, the ability to simulate arbitrary surrogate functions without significant loss in performance as well as its speedup in the backward pass.
\end{itemize}

\section{Background}
The SNN consists of leaky integrate and fire neurons (LIF) that are governed by differential equations in continuous time \cite{gerstner2014neuronal}. They are generally approximated by discrete dynamics given in the form of recurrent equations,
\begin{align}
	u^{(l)}_{i}[t] &= \beta u^{(l)}_{i}[t-1] + \sum_{j} w_{ij} x^{(l-1)}_{j}[t] - x^{(l)}_{i}[t-1] u_{th},\nonumber\\
	x^{(l)}_{i}[t] &= h(u^{(l)}_{i}[t] - u_{th} ) = \begin{cases} 1 & \text{if } u^{(l)}_{i}[t] > u_{th} \\
		0 & \text{otherwise,}	
	\end{cases}
 \label{eq:lif_discrete}
\end{align}
where $u^{(l)}_{i}[t]$ denote the membrane potential of $i$-th neuron in the layer $l$ at time-step (discrete) $t$, which recurrently depends upon its previous potential (with scaling factor $\beta < 1$) and spikes $x^{(l-1)}_{j}[t]$ received from the neurons of previous layers weighted by $w_{ij}$. The neuron generates binary spike $x^{(l)}_{i}[t]$ whenever the membrane potential exceeds threshold $u_{th}$, represented by the Heaviside function $h$, followed by a reset effect on the membrane potential. 

To implement the back-propagation of training loss through the network, one must obtain a derivative of the spike function, which poses a significant challenge in its original form represented as:
\begin{align}
    \frac{dx_i[t]}{du} = \begin{cases}
	\infty & \text{if } u_i^{(l)}[t]= u_{th}\\
	0 & \text{otherwise.}
\end{cases}
\end{align}
where we denote $u:= u_i^{(l)}[t] - u_{th}$. To avoid the entire gradient becoming zero, known as the dead neuron problem, the surrogate gradient method (referred as \sur) redefines the derivative using a surrogate:   
\begin{align}
    \frac{dx_i[t]}{du}:=g(u)
\end{align}
Here, the function $g(u)$ can be, for example, the derivative of the Sigmoid function (see section \ref{sec:surr_to_dist}). To reduce the computational burden of the backward pass, \cite{nieves2021sparse} (referred as \spgd) computes the surrogate gradients only when membrane potential $u_i[t]$ is close to the membrane threshold $u_{th}$.
\begin{align}
    \frac{dx_i[t]}{du} = \begin{cases} g(u_i[t] - u_{th} ) & \text{if, } \abs{u_i[t] - u_{th}} < B_{th}\\
0 & \text{otherwise }
\end{cases}
\end{align}
They introduce a hyper-parameter $B_{th}$ called back-propagation threshold, which controls the fraction of neurons with a non-zero gradient. Fewer neurons participating in the back-propagation can be translated to reduced computational requirement, implying energy savings. However, setting it close to $0$ implies very few active neurons hence no learning but most energy saving, while setting it to $\infty$ enables full surrogate training. Thus, $B_{th}$ determines the energy vs. accuracy trade-off of the \spgd algorithm.  

To address the differentiability issue of SNNs, we intend to employ the zeroth order technique, a popular gradient-free method \cite{liu2020primer}. Consider a function $f: \mathbb{R}^d \rightarrow \mathbb{R}$, that we intend to minimize using gradient descent, for which the gradient may not be available or even undefined. The zeroth-order method estimates the gradients using function outputs.  Given a scalar $\delta > 0$, the 2-point ZO is defines,
\begin{align}
G^{2}(\m w; \m z, \delta) = \phi(d) \frac{f(\m w+\delta \m z )-f(\m w-\delta \m z )}{2\delta} \m z    
\end{align}
where, $\m z \sim \lambda$ is a random direction with $\E{\norm{\m z}^2}{z \sim \lambda}=1$ and $\phi(d)$ is a dimension dependent factor, with $d$ being the dimension. However, in order to approximate the full gradient of $f$ up to a constant squared error, we need an average of $O(d)$ samples of $G^2$, which becomes computationally challenging when $d$ is large, such as the number of learnable parameters of the neural network. Though well studied in the literature, properties of 2-point ZO are known only for the continuous functions \cite{nesterov2017random, berahas2022theoretical}. In the present context, we apply it to the Heaviside function that having a jump discontinuity requires us to derive its theoretical properties.



\section{The \lzo algorithm}
Applying ZO on a global scale is costly due to the large dimensionality of neural networks. Since the non-differentiability of SNN is introduced by the Heaviside function, we apply the 2-point ZO method on $h: \real{} \rightarrow \{0,1\}$ itself,
\begin{align}
\label{eq:lzo}
    G^{2}( u; z, \delta) &= \frac{h(u +z\delta )-h(u-z\delta )}{2\delta}z \nonumber\\
    &= \begin{cases}
	0, & \abs{u} > \abs{z}\delta\\
	\frac{\abs{z}}{2 \delta}, & \abs{u} < \abs{z}\delta\\
\end{cases}
\end{align}

where $u= u_i^{(l)}[t] - u_{th}$ and $z$ is sampled from some distribution $\lambda$. We may average the 2-point ZO gradient over a few samples $z_k$, so that the \lzo derivative of the spike function is defined as:
\begin{equation}
    \frac{dx_i[t]}{dt}:=\frac{1}{m}\sum_{k=1}^{m} G^{2}(u; z_k, \delta)
    \label{eqn:lzo_sum}
\end{equation}
We implement this at the neuronal level of the back-propagation routine, where the forward pass uses the Heaviside function, and the backward pass uses equation \eqref{eqn:lzo_sum}. Note that the gradient $\frac{dx_i[t]}{dt}$ being non-zero naturally determines the active neurons of the backward pass, which can be inferred from the forward pass through the neuron. Algorithm \ref{algo:lzo} gives an abstract representation of the process at a neuronal level, which hints that the backward call is redundant when the neuron has a zero gradient.  

In the energy-efficient implementation of the back-propagation, the optimization of the network weights takes place in a layer-wise fashion through the unrolling of recurrence of equation \eqref{eq:lif_discrete} with respect to time. As the active neurons of each layer for every time step are inferred from the forward pass, gradients of only active neurons are required to be saved for the backward pass, hence saving the memory and computation requirement of the backward pass. One may refer to \cite{nieves2021sparse} for further details of the implementation framework.


\begin{algorithm}[t]
	\caption{\lzo}
	\label{algo:lzo}
	\begin{algorithmic}
	{\small
            \STATE \hspace{-3.5mm}\textbf{Forward} 
            \REQUIRE potential $u := u_i^{(l)}[t] - u_{th}$, distribution $\lambda$, $\delta, m$
            \STATE sample $z_1, z_2, \cdots z_m \sim \lambda$
            \STATE $grad \leftarrow \frac{1}{m}\sum_{k=1}^{m}\mathbb{I}(\abs{u}< \delta \abs{z_k}) \frac{\abs{z_k}}{2\delta}$
            \IF {$grad \neq 0$} 
                \STATE SaveForBackward($grad$)
            \ENDIF
            \STATE \textbf{return} $\mathbb{I}(u> 0)$
            \vspace{1mm}
            \\\hrule
            \vspace{1mm}
            \STATE \hspace{-3.5mm}\textbf{Backward} \, \COMMENT{Invoked if grad is non-zero}
            \REQUIRE gradient from chain rule: $grad\_input$
            \STATE \textbf{return} $grad\_input * grad$ 
	}
\end{algorithmic}
\end{algorithm}
\section{Theoretical Properties of \lzo}
\subsection{General ZO function}
For the theoretical results around \lzo we consider a more general function than what was suggested by \ref{eq:lzo}, in the form
\begin{equation}\label{eq: zo_general}
G^2(u;z,\delta)=\begin{cases}
    0, \quad |u|> |z|\delta\\
    \frac{|z|^\alpha}{2\delta},\quad |u|\leq |z|\delta,
\end{cases}
\end{equation}
where the new constant $\alpha$ is an integer different from 0, while $\delta$ is a positive real number (so, for example, setting $\alpha=1$ in \eqref{eq: zo_general}, we obtain \eqref{eq:lzo}).

The integer $\alpha$ is somewhat a normalizing constant which allows obtaining many different surrogates as the expectation of function $G^2(u; z,\delta)$ when $z$ is sampled from a suitable distribution. In practice, taking $\alpha=\pm 1$ will suffice to account for most of the surrogates found in the literature. The role of $\delta$ is rather different, as it controls the ``shape'' of the surrogate (narrowing it and stretching around zero). The role of each constant will be more clear from what follows (see section \ref{sec:surr_to_dist}).
\subsection{Surrogate functions}
\begin{defn}
\label{def:surr}
We say that a function $g:\R\to \R_{\geq 0}$ is a surrogate function (gradient surrogate) if it is even, non-decreasing on the interval $(-\infty,0)$ and $c:=\int_{-\infty}^\infty g(z)dz<\infty$. 
\end{defn}
Note that the integral $\int_{-\infty}^\infty g(z)dz$ is convergent (as $g(z)$ is non-negative), but possibly can be $\infty$ and the last condition simply means that the function $\frac{1}{c}g(t)$ is a probability density function. The first two conditions, that is, requirements for the function to be even and non-decreasing are not essential but rather practical and are in consistency with examples from SNN literature. 

Note that the function $G:\R\to [0,1]$, defined as $G(t):=\frac{1}{c}\int_{-\infty}^t g(z)dz$ is the corresponding cumulative distribution function (for PDF $\frac{1}{c}g(t)$). Moreover, it is not difficult to see that its graph is ``symmetric'' around point $(0,\frac{1}{2})$ (or in more precise terms, $G(t)=1-G(-t)$), hence $G(t)$ can be seen as an approximation of Heaviside function $h(t)$. Then, its derivative $\frac{d}{dt}G(t)=\frac{1}{c}g(t)$ can serve as an approximation of the ``derivative'' of $h(t)$, or in other words, as its surrogate, which somewhat justifies the terminology. 

Finally, one may note that ``true'' surrogates would correspond to those functions $g$ for which $c=1$. However, the reason we allow $c$ to be different from 1 is again practical and simplifies the derivation of the results that follow. We note once again that allowing general $c$ is in consistency with examples used in the literature.
\subsection{Surrogates and ZO}
To be in line with classic results around ZO method and gradient approximation of functions, we pose ourselves two basic questions: What sort of functions in variable $u$ can be obtained as the expectation of $G^2(u;z,\delta)$ when $z$ is sampled from a suitable distribution $\lambda$, and, given some function $g(u)$, can we find a distribution $\lambda$ such that we obtain $g(u)$ in the expectation when $z$ is sampled from $\lambda$? Two theorems that follow answer these questions and are the core of this section.

The main player in both of the questions is the expected value of $G^2(u;z,\delta)$, so we start by analyzing it more precisely. Let $\lambda$ be a distribution, $\lambda(t)$ its PDF for which we assume that it is even and that $\int_0^{\infty}z^\alpha\lambda(z)dz<\infty$. Then, we may write
\begin{align}
\E{&G^{2}(u; z, \delta)}{z \sim \lambda }=\int\limits_{-\infty}^{\infty}G^{2}(u; z, \delta)\lambda(z)dz\nonumber \\
&= \int\limits_{|u|\leq |z|\delta}\frac{|z|^\alpha}{2\delta}\lambda(z)dz  =\frac{1}{\delta}\int\limits_{\frac{|u|}{\delta}}^\infty z^\alpha \lambda(z) dz. \label{eq: expected zo}
\end{align}
Then, it becomes apparent from \eqref{eq: expected zo} that $\E{G^{2}(u; z, \delta)}{z \sim \lambda }$ has some properties of surrogate functions (it is even, and non-decreasing on $\R_{<0}$). The proofs of the following results are detailed in the appendix. 
\begin{restatable}{lem}{lemmaone}
    Assume further that $\int_{0}^{\infty}z^{\alpha+1}\lambda(z)dz<\infty$. Then, $\E{G^{2}(u; z, \delta)}{z \sim \lambda }$ is a surrogate function.
\end{restatable}
\begin{restatable}{thm}{thmone}
\label{thm: expectation}
Let $\lambda$ be a distribution and $\lambda(t)$ its corresponding PDF. Assume that integrals $\int_0^\infty t^\alpha \lambda(t)dt$ and $\int_0^\infty t^{\alpha+1}\lambda(t)dt$ exist and are finite. Let further $\tilde{\lambda}$ be the distribution with corresponding PDF function
$$
\Tilde{\lambda}(z) = \frac{1}{c} \int\limits_{|z|}^\infty t^\alpha \lambda(t) dt,
$$
where $c$ is the scaling constant (such that $\int_{-\infty}^\infty \tilde{\lambda}(z)dz=1$). 
Then,
$$
\E{G^{2}(u; z, \delta)}{z \sim \lambda }=\frac{d}{du} \E{c\, h(u+\delta z)}{z \sim \tilde{\lambda} }.
$$
\end{restatable}
For our next result, which answers the second question that we asked at the beginning of this section, note that a surrogate function is differentiable almost everywhere, which follows from Lebesgue theorem on differentiability of monotone functions. So, taking derivatives here is understood in an ``almost everywhere'' sense.

\begin{restatable}{thm}{thmtwo}
\label{thm: main2}
    Let $g(u)$ be a surrogate function. Suppose further that $c=-2\delta^2\int_{0}^\infty \frac{1}{z^\alpha}g'(z\delta)dz <\infty$ and put $\lambda(z) = -\frac{\delta^2}{c z^\alpha}g'(z\delta)$ (so that $\lambda(z)$ is a PDF). Then,
    $$c\,\E{G^2(u;z,\delta)}{z\sim \lambda} = \E{c\,G^2(u;z,\delta)}{z\sim \lambda}  = g(u).
    $$
\end{restatable}

\subsection{Obtaining full-surrogates on Expectation} \label{sec: full-surrogates}
\label{sec:dist_to_sur}
In the next sections we spell out the results of Theorem \ref{thm: expectation} applied to some standard distributions, with $\alpha=1$. For clarity, all the parameters of the distributions are chosen in such a way that the scaling constant of the resulting surrogate is 1. One may consult Figure \ref{fig:lzo_distribution} for the visual representation of the results, while the details are provided in the appendix.

\subsubsection{From standard Gaussian}
Recall that the standard normal distribution $N(0,1)$ has PDF of the form $\frac{1}{\sqrt{2\pi}}\exp(-\frac{z^2}{2})$. Consequently, it is straightforward to obtain
\begin{align}
\E{G^{2}(u; z, \delta)}{z \sim \lambda } &=  \frac{1}{\sqrt{2 \pi}}\int_{-\infty}^{\infty} \frac{\abs{z}}{2 \delta} \exp(-\frac{z^2}{2}) dz \nonumber \\
&=\frac{1}{\delta \sqrt{2\pi}} \exp(-\frac{u^2}{2\delta^2}).
\end{align} 
\subsubsection{From Uniform Continuous}
Consider the PDF of a continuous uniform distribution: 
$$f(z; a, b) = \begin{cases}
    \frac{1}{b-a} & \text{for} \, z \in [a,b]\\
    0    & \text{otherwise},
\end{cases}$$ 
where $a<b$ are some real numbers. For the distribution to be even and the resulting scaling constant of the surrogate to be 1 (which translates to $\E{z}{}=0$ and $\E{z^2}{}=1$, respectively) we set, $a=-\sqrt{3}$, $b=\sqrt{3}$. Then,
\begin{align}
\E{G^{2}(u; z, \delta)}{z \sim \lambda }=  \int_{-\infty}^{\infty}  \frac{\abs{z}}{2\delta} f(z) dz \nonumber \\
    = \begin{cases}
        \frac{1}{4\sqrt{3}\delta } (3 - \frac{u^2}{\delta^2}) & \text{if } \frac{\abs{u}}{\delta} < \sqrt{3}, \\
        0 & \text{otherwise.}
    \end{cases}
 \end{align}

\subsubsection{From Laplacian Distribution}
The PDF of Laplace distribution is given by:
$$f(z; \mu, b) = \frac{1}{2b} \exp(-\frac{\abs{z-\mu}}{b})$$ 
with mean $\mu$ and variance $2b^2$. Setting, $b=\frac{1}{\sqrt{2}}$ and $\mu=0$
and using \eqref{eq:lzo} we obtain,
\begin{align}
    &\E{G^{2}(u; z, \delta)}{z \sim \lambda }
    =\frac{2}{\sqrt{2}}\int_{\frac{\abs{u}}{\delta}}^{\infty}  \frac{\abs{z}}{2 \delta} \exp(-\sqrt{2}\abs{z})dz\nonumber\\ 
    &= \frac{1}{2\delta} (\frac{\abs{u}}{\delta}+\frac{1}{\sqrt{2}})  \exp(-\sqrt{2}\frac{\abs{u}}{\delta}).
 \end{align}

\begin{figure*}
\begin{center}
\centerline{
\includegraphics[width=0.32 \textwidth]{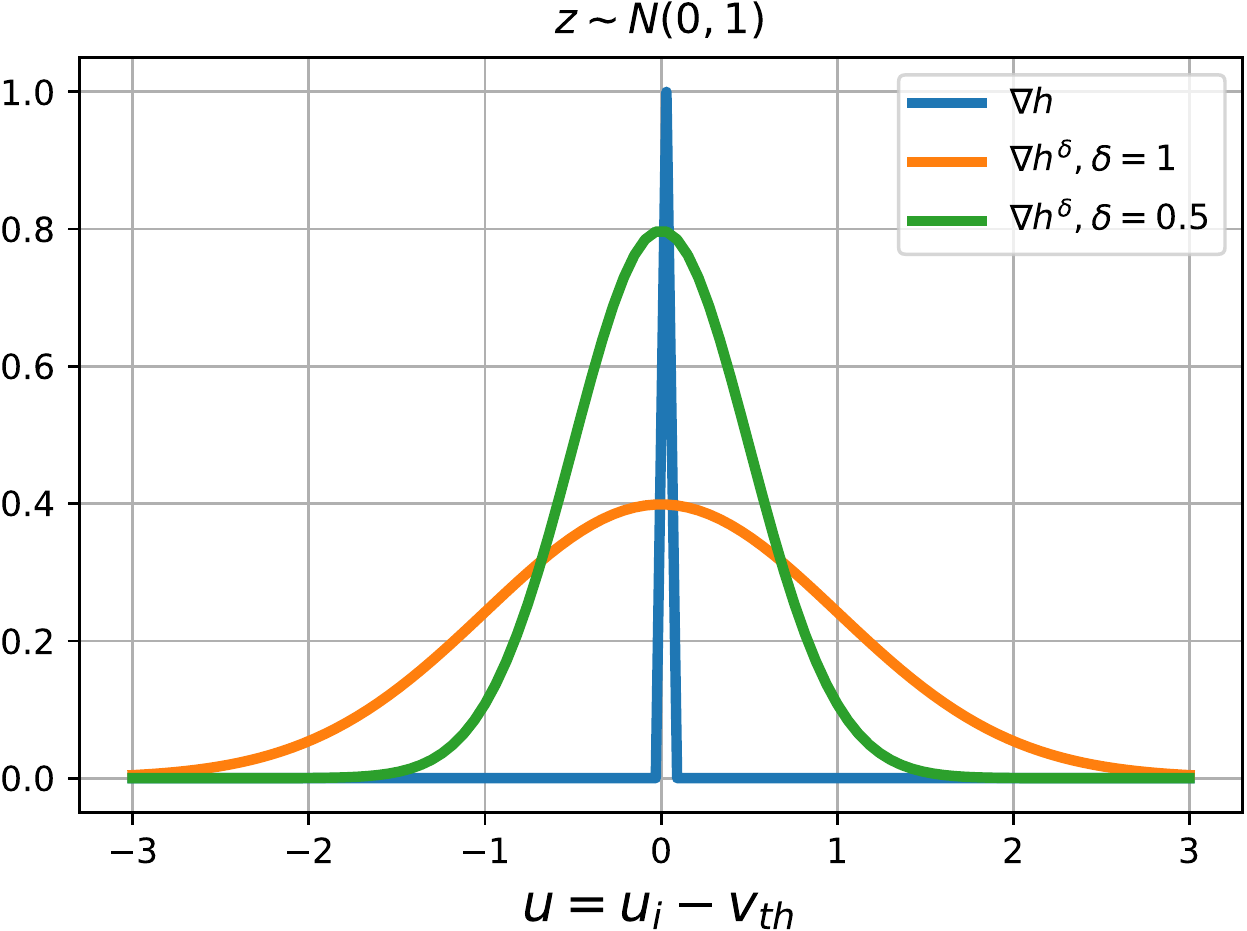}
\includegraphics[width=0.32 \textwidth]{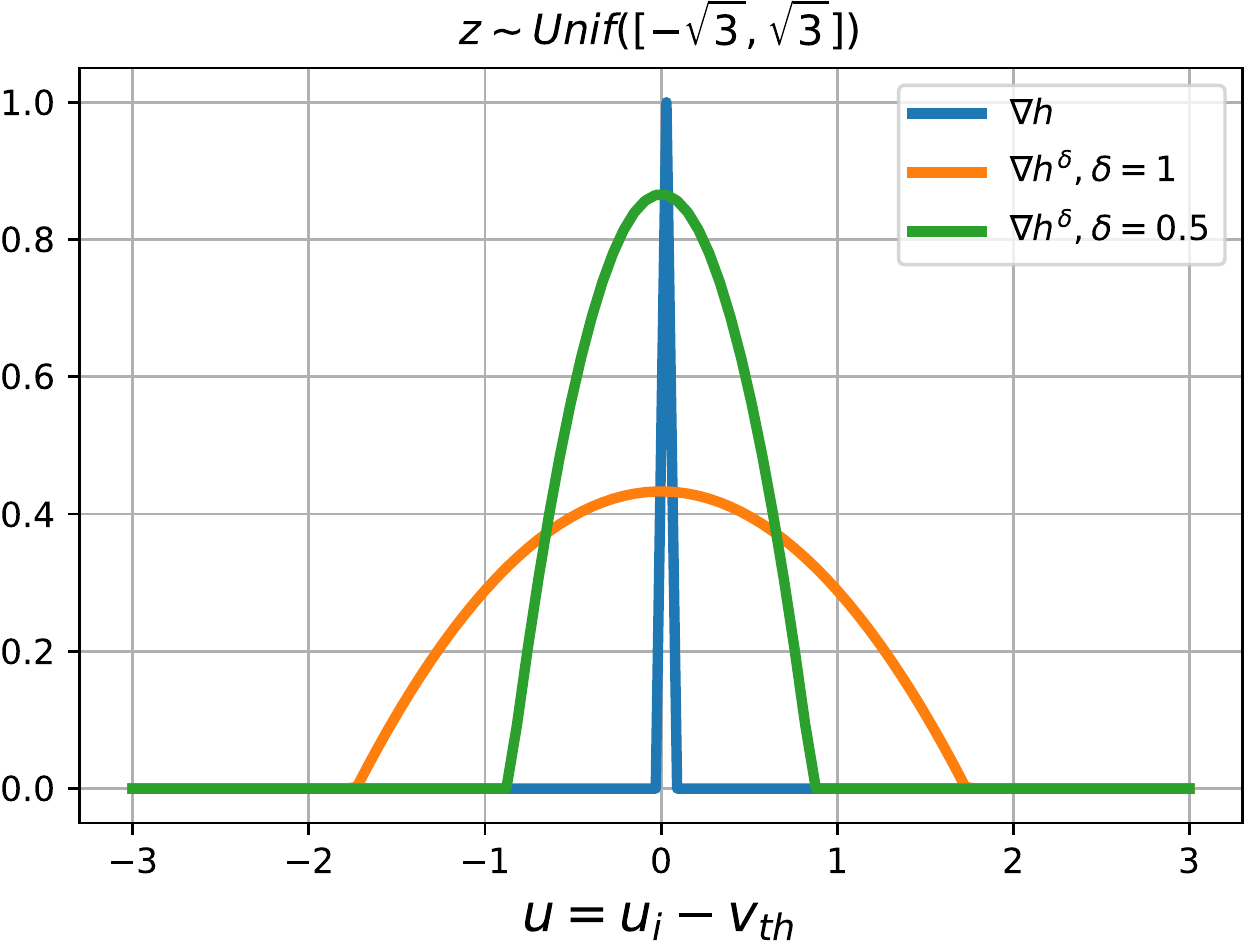}
\includegraphics[width=0.32 \textwidth]{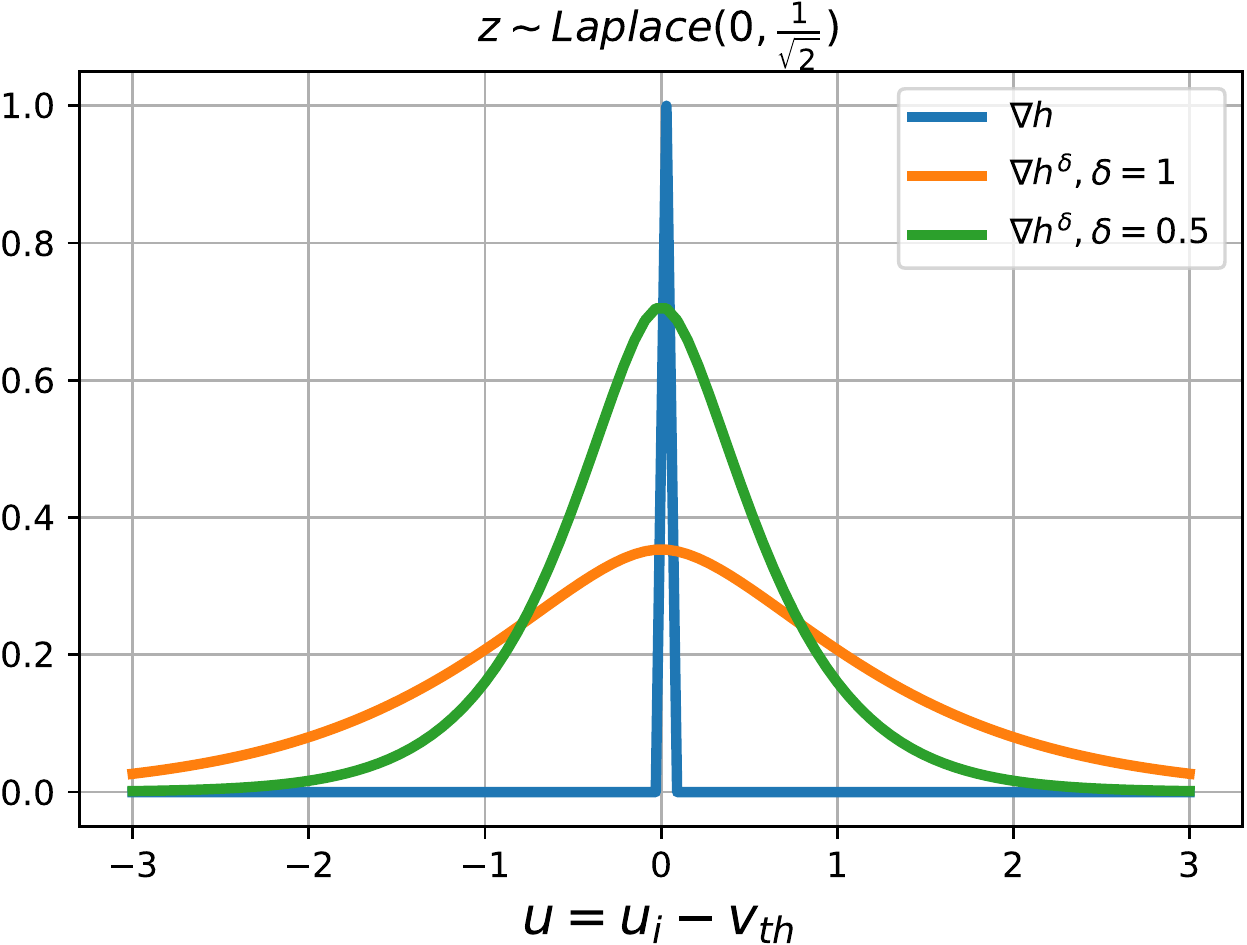}}
\caption{The figure shows the expected surrogates derived in section \ref{sec:dist_to_sur} as $z$ is sampled from Normal$(0,1)$, Unif$([\sqrt{3}, \sqrt{3}])$ and Laplace$(0, \frac{1}{\sqrt{2}})$ respectively. Each figure shows the surrogates corresponding to $\delta \rightarrow 0$, $\delta=0.5$ and $1$. The surrogates are used by \sur and \spgd methods for a fair comparison with \lzo as the latter uses respective distributions to sample $z$.}
\label{fig:lzo_distribution}
\end{center}
\vskip -0.3in
\end{figure*}

\subsection{Expected back-propagation threshold for \lzo}
\label{sec:expected_th}
One of the key properties of \lzo is that in the backward pass it is able to simulate arbitrary surrogate functions, but still be energy efficient. The latter is visible in equations \eqref{eq:lzo} and \eqref{eq: zo_general} as $G^2(u_i^{(l)}[t] - u_{th};z,\delta)$ is zero whenever $|u_i^{(l)}[t] - u_{th}|>\abs{z}\delta$. The role of the back-propagation threshold is then played by the value $\abs{z}\delta$, and since $z$ is sampled from a suitable distribution, we study what is the expected value of this quantity. 


In what follows, $m$ is the number of samples used in \eqref{eqn:lzo_sum}, while $k$ is the index of a particular sample. To compute the expected back-propagation threshold, we observe that a neuron is inactive in \lzo back-propagation if,
\begin{align*}
 \abs{u_i^{(l)}[t] - u_{th}} > \abs{z_k}\delta, \quad \text{for $k=1,\dots,m$}, \\   
\text{or,} \abs{u_i^{(l)}[t] - u_{th}} > t\delta, \,\text{where} \, t=\max_k\{\abs{z_1}, \cdots,\abs{z_m}\}
\end{align*}
 Assume $z_k \sim \lambda$, where $\lambda(t)$ denotes the PDF of the sampling distribution, with the corresponding CDF denoted by $F_{z_k}$. The PDF, $\tilde{\lambda}$, of the random variable $\abs{z_k}$ is given by
\begin{align}
\tilde{\lambda}(x)=\begin{cases}
   0, & \text{if } x <0 \\
   2\lambda(x), & \text{otherwise.}
\end{cases}    
\end{align}

The corresponding CDF is obtained by integrating the previous expression,
\begin{align}
F_{\abs{z_{k}}}(x)=\begin{cases}
   0, & \text{if } x <0 \\
   2(F_{z_k}(x) - F_{z_k}(0)), & \text{otherwise.}
\end{cases}
\end{align}
Further note that, 
\begin{align}
    F_{t}(x)=P(t < x)=\prod_{k=1}^{m} P(\abs{z_k} < x) = F_{\abs{z_k}}^m(x)
\end{align}
If we denote the PDF of the random variable $t$ as $\hat{\lambda}$, we obtain
\begin{align}
\hat{\lambda}(x)= m F_{\abs{z_k}}^{m-1}(x) \tilde{\lambda}(x).    
\end{align}
Finally, the expected back-propagation threshold takes the form 
\begin{align}
\tilde{B}_{th}=\delta\E{t}{} = \delta \int_{0}^{\infty} t \hat{\lambda}(t)dt.
\end{align}
In the cases of distributions used in experimental sections, the previous expression simplifies. Table \ref{tab:expected_th} gives the numerical values for some particular $m$. (See appendix for details.)
\begin{table}
\caption{Computing the expected back-propagation thresholds}
\label{tab:expected_th}
\begin{center}
\begin{tabular}{cccc} 
\toprule
\multicolumn{2}{c}{$z \sim \lambda$ } & \multicolumn{2}{c}{$\tilde{B}_{th}/\delta $ }\\
\hline
$\lambda$& $F_{\abs{z_k}}(x)$ & 	$m=1$ & $m=5$  \\
\hline
Normal$(0,1)$ &$\erf(\frac{x}{\sqrt{2}})$ &$0.798$ & $1.569$ \\
Unif$([\sqrt{3}, \sqrt{3}])$ & $\frac{x}{\sqrt{3}}$ &$0.866$ & $1.443$ \\
Laplace$(0, \frac{1}{\sqrt{2}})$& $1-\exp(-\sqrt{2} x)$ &$0.707$ & $1.615$\\
\bottomrule
\end{tabular}    
\end{center}
\vskip -0.2in
\end{table}
\vspace{-5mm}
\subsection{Simulating a specific Surrogate}\label{sec:surr_to_dist}
In what follows we use Theorem \ref{thm: main2} to derive corresponding distributions for some other surrogate functions. Complete derivations are delayed to the appendix.
\subsubsection{Sigmoid}
Consider the Sigmoid surrogate function, where the Heaviside is approximated by the differentiable Sigmoid function \cite{zenke2018superspike}. The corresponding surrogate gradient is given by,
\begin{align*}
   \frac{dx}{du} = \frac{d}{du}\frac{1}{1 + \exp(-k u)} = \frac{k \exp(-k u)}{(1 + \exp(-ku))^2} =: g(u)
\end{align*}
Note that $g(u)$ satisfies our definition of a surrogate (being even, non-decreasing on $(-\infty,0)$ and 
$\int_{-\infty}^{\infty} g(u) du = 1<\infty$).
Thus, according to Theorem \ref{thm: main2}, the corresponding PDF is given by
\begin{align}
    \lambda(z)=-\frac{\delta^2}{c}\frac{g'(\delta z)}{z} =a^2\frac{\exp(-k \delta z)(1-\exp(-k\delta z))}{z(1+\exp(-k\delta z))^3}
\end{align}
where, $c=\frac{\delta^2 k^2}{a^2}$ and $a:=\sqrt{\frac{1}{0.4262}}$.
The temperature parameter, $k$, comes from the surrogate to be simulated, while $\delta$ is introduced by \lzo. The expected back-propagation threshold for $m=1$, $\tilde{B}_{th} = \delta \, \E{\abs{z}}{z\sim \lambda }$, with,
\begin{align*}
\E{\abs{z}}{z\sim \lambda } &= 2a^2 \int_{0}^{\infty} z  \frac{\exp(-az)(1-\exp(-az))}{z(1+\exp(-az))^3} dz = \frac{a}{2}.
\end{align*}
\subsubsection{Fast Sigmoid}
Consider also the Fast Sigmoid surrogate \cite{zenke2018superspike, nieves2021sparse} that avoids computing the exponential function in Sigmoid to obtain the gradient: 
\begin{align*}
   \frac{dx}{du} = \frac{1}{(1 + k\abs{u})^2} =: g(u).
\end{align*}
We choose $\alpha=-1$ (note that $\alpha=1$ does not work in this case) and apply theorem \ref{thm: main2} so that the PDF is then given by
\begin{align}
\lambda(z)= -\frac{1}{c} \frac{\delta^2}{z^\alpha}g'(z \delta) = k^2\delta^2  \frac{z \sign(z \delta)}{(1 + k\abs{z \delta})^3}. 
\end{align}
and derive $c=\frac{2}{k}. $We may verify, $\int_{-\infty}^{\infty}\lambda(z)dz=1$,  however, the expected back-propagation threshold, $\tilde{B}_{th}=\delta\E{\abs{z}}{z\sim \lambda }$ does not converge. This implies that direct energy comparison with \spgd for the Fast Sigmoid is not possible, but the way around is to use Fast Sigmoid surrogate with finite support, which is an automatic consequence of the practical implementation of inverse transform sampling discussed in the next section. 
\subsubsection{Inverse Transform Sampling}\label{ssec: inverse transform}
To simulate a given surrogate in \lzo, one needs to sample from the corresponding distribution described by the PDF $\lambda$.  Given a sample $r \sim$ Unif$([0,1])$ and the inverse CDF $\Lambda^{-1}$ of the distribution, the inverse sampling technique returns, $\Lambda^{-1}(r)$, as  a sample from the distribution. If the inverse CDF is not computable analytically from the PDF (or not implementable practically), we may choose a finite support over which the PDF is evaluated at a sufficiently dense set of points, and compute the discretized CDF using the Riemann sum. The inverse discretized CDF is then computed empirically and stored as a list for a finite number of points (spaced regularly) between $[0,1]$. Sampling from the uniform distribution then amounts to randomly choosing the indices of the list and picking the corresponding inverse CDF values.

\section{Experiments} 
\begin{figure*}[!t]
\begin{center}
\centerline{\includegraphics[width=0.33 \textwidth]{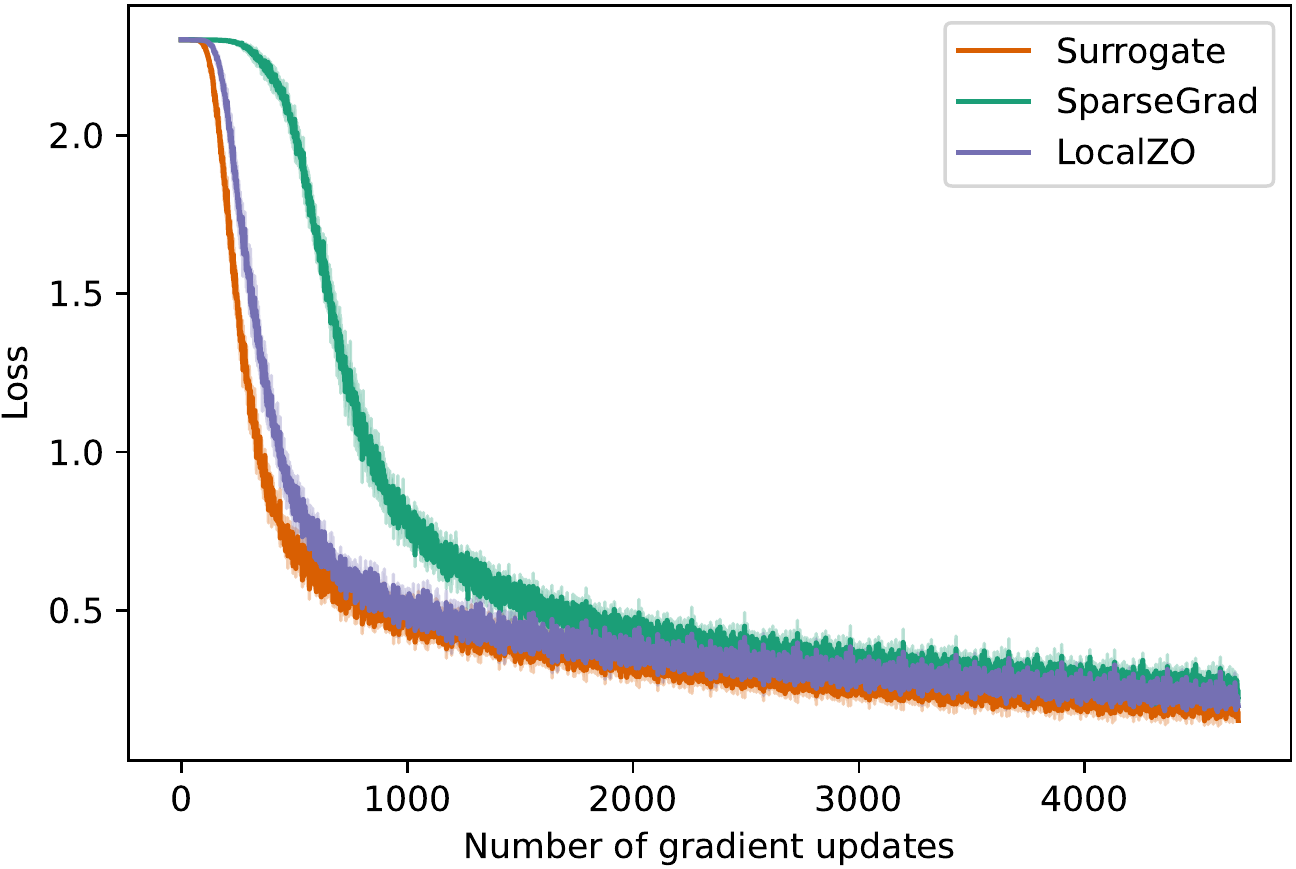}
\includegraphics[width=0.33 \textwidth]{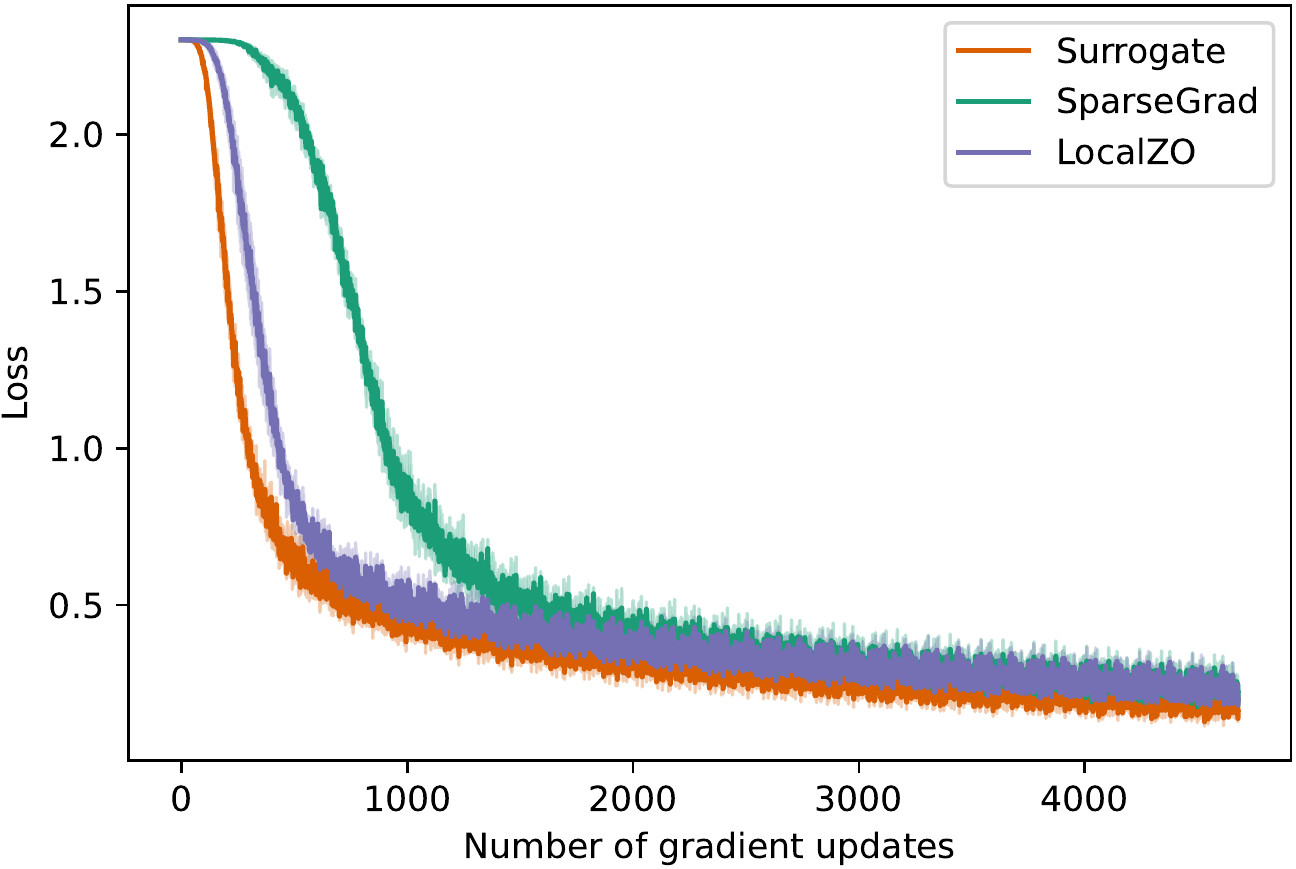}
\includegraphics[width=0.33 \textwidth]{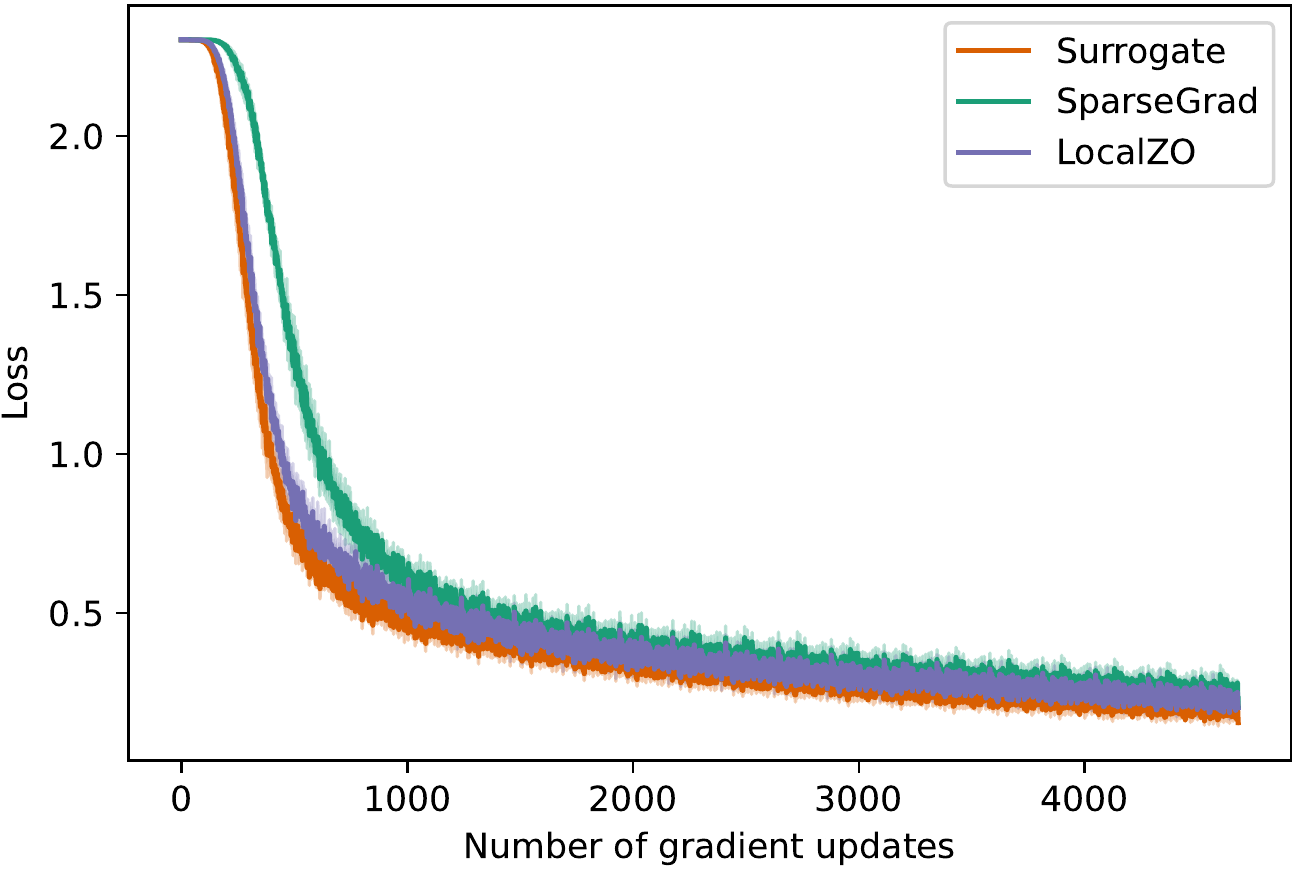}}
\caption{We plot the training loss of the algorithms reported in Table \ref{tab:dist} after each gradient update, for distributions Normal$(0,1)$, Laplace$(1, \frac{1}{\sqrt{2}})$, and Unif$([-\sqrt{3}, \sqrt{3}])$ respectively, for $\delta=0.05$ and $m=1$. The \lzo algorithm converges faster than the \spgd method, resulting in better generalization.}
\label{fig:loss_dist}
\end{center}
\vskip -0.4in
\end{figure*}
\begin{table*}
\caption{Performance comparison over distributions and the corresponding surrogates in NMNIST}
\label{tab:dist}
\begin{center}
\begin{footnotesize}
\begin{sc}
\begin{tabular}{lcccccccc}
\toprule
Method & Train & Test & Back. & Over. & Train & Test & Back. & Over.  \\
\midrule
&\multicolumn{4}{c}{$z \sim$ Normal$(0,1)$, $\delta=0.05, m=1$ }&\multicolumn{4}{c}{$z \sim$ Normal$(0,1)$, $\delta=0.05, m=5$ }\\
\midrule
\sur &95.25 $\pm$ 0.14& 93.70$\pm$ 0.10& 1 & 1 &95.44 $\pm$ 0.22& 93.76$\pm$ 0.10& 1 & 1 \\
\spgd &93.26 $\pm$ 0.31& 91.86$\pm$ 0.29& 99.57 & 3.38 &95.02 $\pm$ 0.29& 93.39$\pm$ 0.25 & 80.0& 3.40\\
\lzo   &94.38 $\pm$ 0.12& 93.29$\pm$ 0.08& 92.27 & 3.34 &95.20 $\pm$ 0.22& 93.69$\pm$ 0.17& 77.7 & 3.22 \\
\midrule
&\multicolumn{4}{c}{$z \sim$ Laplace$(1, \frac{1}{\sqrt{2}})$, $\delta=0.05, m=1$ }&\multicolumn{4}{c}{$z \sim$ Laplace$(1, \frac{1}{\sqrt{2}})$, $\delta=0.05, m=5$ }\\ %
\midrule
\sur   &95.61 $\pm$ 0.16& 93.76$\pm$ 0.08& 1 & 1 &95.42 $\pm$ 0.03& 93.73$\pm$ 0.04& 1 & 1 \\
\spgd  &93.97 $\pm$ 0.43& 92.65$\pm$ 0.52& 88.2 & 3.19 &94.73 $\pm$ 0.29& 93.13$\pm$ 0.23& 72.9 & 3.15\\
\lzo     &94.25 $\pm$ 0.17& 93.05$\pm$ 0.09& 83.7 & 3.07 &95.07 $\pm$ 0.03& 93.63$\pm$ 0.05& 69.4 & 2.80 \\
\midrule
&\multicolumn{4}{c}{ $z \sim$ Unif$([-\sqrt{3}, \sqrt{3}])$, $\delta=0.05, m=1$ }&\multicolumn{4}{c}{ $z \sim$ Unif$([-\sqrt{3}, \sqrt{3}])$, $\delta=0.05, m=5$ }\\ 
\midrule
\sur   &95.15 $\pm$ 0.19& 93.74$\pm$ 0.09& 1 & 1 &95.15 $\pm$ 0.19& 93.74$\pm$ 0.09& 1 & 1\\
\spgd &93.34 $\pm$ 0.44& 91.85$\pm$ 0.35& 83.2 & 3.26 &94.82 $\pm$ 0.27& 93.38$\pm$ 0.17& 76.4 & 3.14 \\
\lzo &94.24 $\pm$ 0.46& 93.05$\pm$ 0.37 & 84.8 & 3.43 &94.95 $\pm$ 0.35& 93.47$\pm$ 0.23 & 73.5 & 2.91 \\
\bottomrule
\end{tabular}
\end{sc}
\end{footnotesize}
\end{center}
\vskip -0.1in
\end{table*}
\begin{figure*}[!t]
\begin{center}
\centerline{\includegraphics[width=0.32\linewidth]{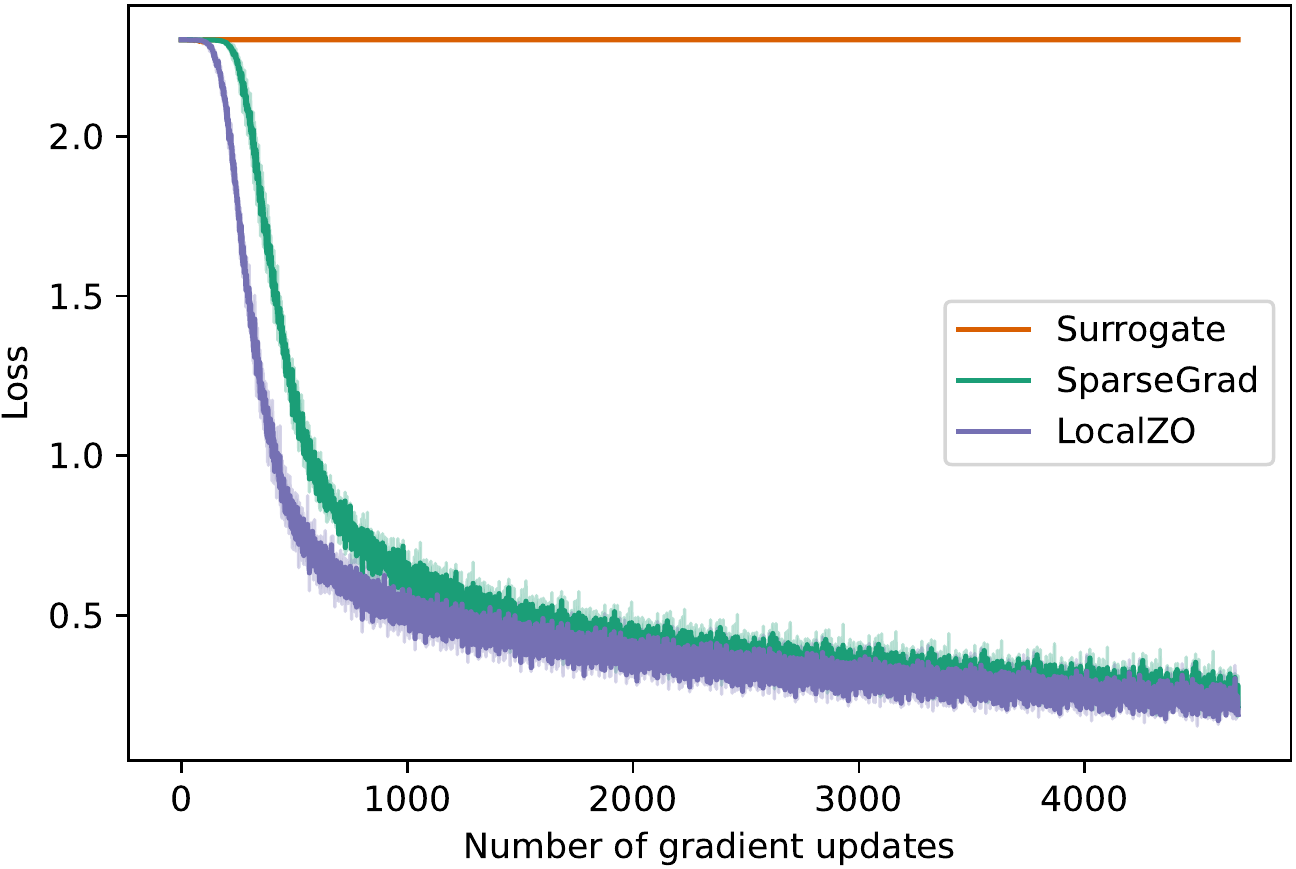}
\includegraphics[width=0.32\linewidth]{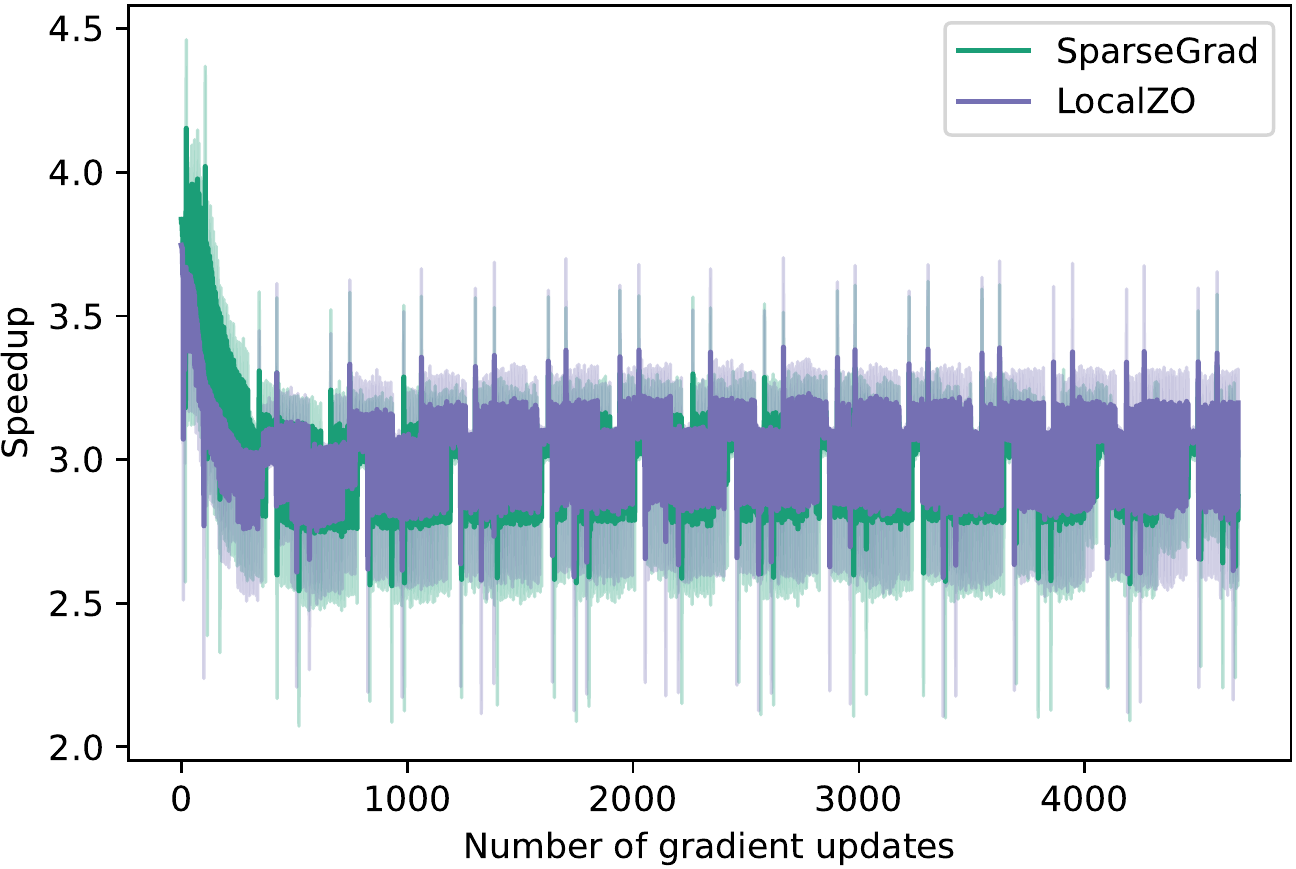}
\includegraphics[width=0.32\linewidth]{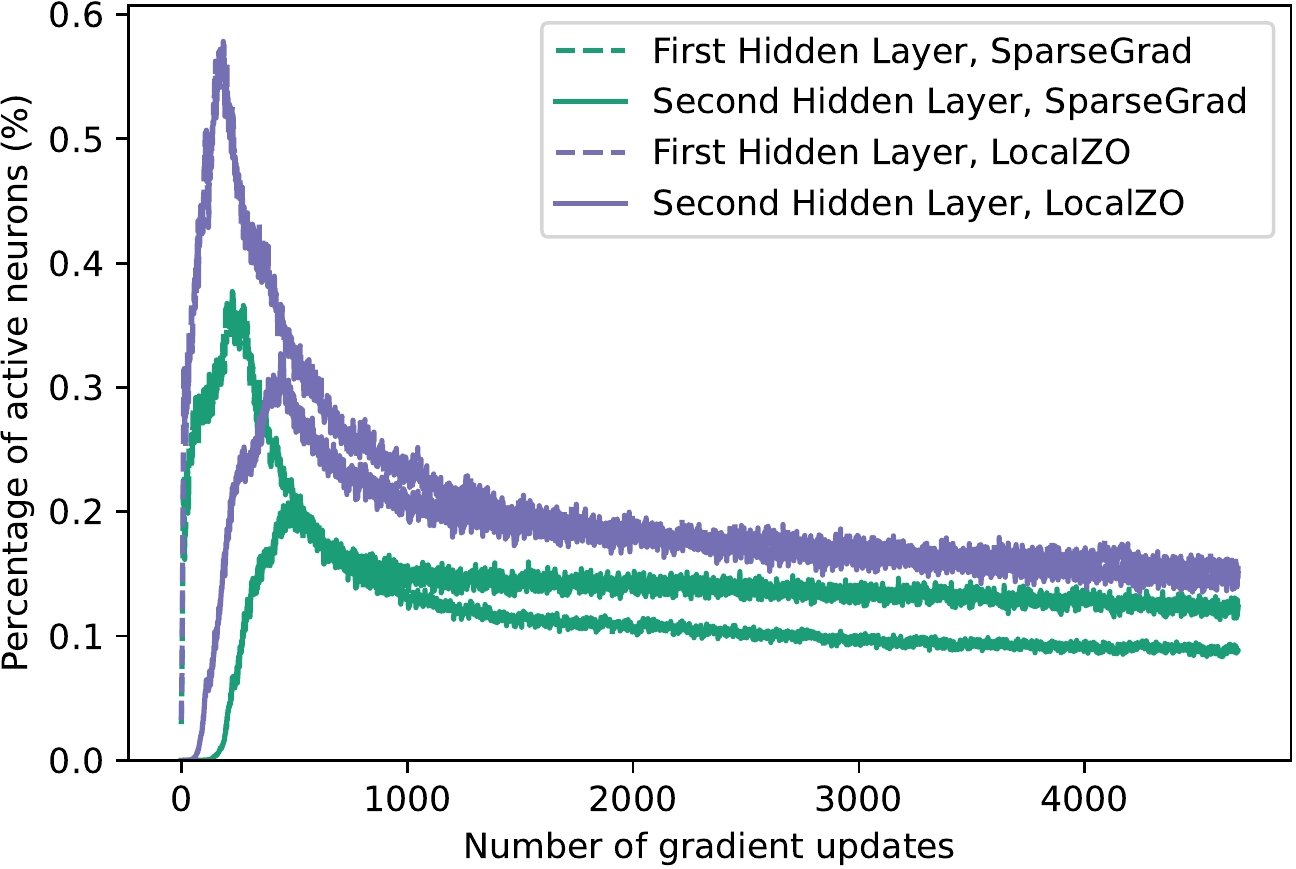}}
\caption{We plot training loss, overall speedup, and percentage of active neurons for the Sigmoid surrogate, as reported in Table \ref{tab:surr}. The \lzo algorithm converges faster than the \spgd method, while having the similar overall speedup. The percentage of active neurons being less 0.6\% explains the reduced computational requirement which translates to backward speedup.}
\label{fig:sigmoid}
\end{center}
\vskip -0.3in
\end{figure*}
We perform classification tasks on popular neuromorphic datasets and compare the performance of the algorithms. The experiments are carried out on an NVIDIA RTX A6000 GPU, with Pytorch CUDA extension.\\
\textbf{Datasets} The Neuromorphic-MNIST or NMNIST \cite{orchard2015converting} is a popular neuromorphic dataset, where static images of handwritten digits between zero and nine are converted to temporal spiking data using visual neuromorphic sensors. The data labeled into ten classes simulates the neuronal inputs captured by the biological visual sensors.  
The Spiking Heidelberg Digits (SHD) \cite{cramer2020heidelberg} is another neuromorphic audio dataset consisting of spoken digits between zero to nine in English and German language amounting to twenty class labels. To challenge the generalizability of the network in practical tasks, 81\% of test inputs of this dataset are new voice samples that are not present in the training data. Finally, the Fashion-MNIST (FMNIST) \cite{xiao2017fashion} dataset uses temporal encoding to convert static gray-scale images based on the principle that each input neuron spikes only once, and a higher intensity spike results in an earlier spike. 

\textbf{Network Architecture and Hyper-parameters} We use fully connected LIF neural network with two hidden layers of 200 neurons each, along with input and output layers, following the \spgd method. We also use the same hyper-parameters mentioned in their work, except for those which we mention explicitly.

\textbf{Comparison with \sur and \spgd} To compare, we supply the \sur and \spgd method the surrogate approximated by \lzo. The surrogates are described in sections \ref{sec:dist_to_sur} and \ref{sec:surr_to_dist}. The \spgd algorithm also requires a back-propagation threshold parameter, $B_{th}$, to control the number of active neurons participating in the back-propagation. We supply it the expected back-propagation threshold $\tilde{B}_{th}$ of \lzo as obtained in sections \ref{sec:expected_th} and \ref{sec:surr_to_dist}. 

\textbf{Performance Metrics} We train every model for 20 epochs and report the average training and test accuracies computed over 5 trials. We compute the speedup of the energy efficient methods, i.e., \spgd and \lzo,  with respect to the \sur method. The backward speedup captures the number of times the backward pass of a gradient update is faster than that of the \sur method. The speedup reported in the experiments is an average over all the gradient updates and the experimental trials. The overall speedup considers the total time required for the forward and the backward pass and then computes the ratio with that of the \sur method. We also compute the active neurons at each layer as a percentage, normalizing by the batch size, number of neurons in the layer, and the latency. The normalization reflects the computation required by a non-sparse gradient so that the percentage of active neurons serves as the proxy of computational savings due to the energy-efficient implementation.

\subsection{From distributions to surrogates}
We demonstrate performance of \lzo over different distributions of $z$, such as standard Normal, Uniform$([\sqrt{3}, \sqrt{3}])$ and Laplace$(0, \frac{1}{\sqrt{2}})$, for $m\in \{1,5\}$ and $\delta=0.05$. The distributions are of unit variance so that parameter $\delta$ is comparable across the methods. We supply \spgd algorithm the back-propagation threshold $\tilde{B}_{th}$ obtained in Table \ref{tab:expected_th}. Table \ref{tab:dist} shows the performance of the methods on the N-MNIST dataset in terms of accuracy and speedup. The \lzo method obtains better train and test accuracies for all cases with a slight compromise in the speedup, except for the uniform distribution where it offers better speedup for $m=1$ compared to the \spgd method. Though we do not report the percentage of active neurons explicitly, they remain comparable and below 1\% for both methods throughout the experiments, ensuring a reasonable backward speedup.

Figure \ref{fig:loss_dist} shows the training loss for the methods after each gradient update for $m=1$, averaged over five trials. The \lzo loss is consistently closer to the \sur loss, which results in better training and test accuracies. In contrast, the fixed truncation used by \spgd affects its gradients. 

\subsection{From surrogates to distributions}
In the section \ref{sec:surr_to_dist}, we derived distributions corresponding to popular surrogates. Using the inverse transform sampling, we implement \lzo, with $m=1, \delta=0.05$.  For the Sigmoid surrogate, we take the temperature parameter $k=a/\delta\approx 30.63$ so that $c=\frac{\delta^2 k^2}{a^2}=1$ and supply \spgd method the corresponding back-propagation threshold, $\Tilde{B}_{th} = 0.766\delta$.
For the Fast Sigmoid surrogate, we choose $k=100$ following \cite{nieves2021sparse} so that $c=\frac{2}{k}$. To compute the expected back-propagation threshold we consider finite support $[-10,10]$ used in the inverse transform sampling of $z$ and evaluate $\tilde{B}_{th} = 0.0461$.

Table \ref{tab:surr} reports the details of the comparison over the NMNIST dataset. \lzo method obtains better test accuracies with a slight reduction in speedup compared to \spgd. The \sur method did not converge with Sigmoid, due to the numerical computation of the exponential.  The \spgd evaluates gradients only for small membrane potentials, so it does not face the issue.  

Figure \ref{fig:sigmoid} shows the training loss, overall speedup, and percentage of active neurons after each gradient step for the Sigmoid surrogate. The sparseness of active neurons (under 0.6\%) explains the reduced computational requirement that translates to the speedup. 

\begin{table}[!t]
\caption{Simulating Surrogates using a distribution on NMNIST}
\label{tab:surr}
\vskip -0.5in
\begin{center}
\begin{footnotesize}
\begin{sc}
\begin{tabular}{lcccc}
\toprule
Method & Train & Test & Back. & Over. \\
\midrule
\multicolumn{5}{c}{Sigmoid, $\delta=0.05, k \approx 30.63, m=1$ }\\ %
\midrule
\spgdt &92.96$\pm$ 0.26& 91.04$\pm$ 0.32& 87.45 & 3.00\\
\lzot    &93.98$\pm$ 0.08& 92.97$\pm$ 0.05& 83.54 & 3.02\\
\midrule
\multicolumn{5}{c}{FastSigmoid, $\delta=0.05, k=100, m=1$ }\\
\midrule
\surt  &93.33$\pm$ 0.05& 91.20$\pm$ 0.11& 1 & 1\\
\spgdt &93.24$\pm$ 0.23& 92.16$\pm$ 0.20& 84.87 & 3.18\\
\lzot  &93.44$\pm$ 0.13& 92.52$\pm$ 0.09& 73.23 & 3.11\\
\bottomrule
\end{tabular}
\end{sc}
\end{footnotesize}
\end{center}
\vskip -0.3in
\end{table}

\subsection{Comparison in other datasets}
Table \ref{tab:dataset_SHD} provides a comparison of the algorithms over the SHD and FMNIST datasets, using surrogates corresponding to the Normal and Sigmoid, with $\delta = 0.05$ and $m=1$. For the Sigmoid, we use the parameters reported in last section. The \sur and \spgd methods are supplied with the corresponding surrogates and back-propagation thresholds. The difference between training and test accuracies for the SHD dataset can be attributed to the unseen voice samples in the test data\cite{cramer2020heidelberg}. The \lzo method offers better test accuracies than \spgd, with a slight compromise in speedup.  
\begin{table}[!t]
\caption{Performance comparison on SHD and FMNIST}
\label{tab:dataset_SHD}
\begin{center}
\begin{footnotesize}
\begin{sc}
\begin{tabular}{lcccc}
\toprule
Method & Train & Test & Back. & Over. \\
\midrule
\multicolumn{4}{c}{SHD, $z \sim$ Normal$(0,1)$, $\delta=0.05, m=1$ }\\ %
\midrule
\surt  &94.58$\pm$ 0.31& 75.48$\pm$ 0.70& 1 & 1\\
\spgdt &92.03$\pm$ 0.79& 74.73$\pm$ 0.73& 143.7 & 4.83\\
\lzot    &91.77$\pm$ 0.27& 76.55$\pm$ 0.93& 142.8 & 4.75\\
\midrule
\multicolumn{4}{c}{SHD, Sigmoid, $\delta=0.05, k \approx 30.63, m=1$ }\\ %
\midrule
\spgdt &92.19$\pm$ 0.41& 75.80$\pm$ 0.97& 140.8 & 4.46\\
\lzot    &91.96$\pm$ 0.11& 76.97$\pm$ 0.40& 133.6 & 4.36\\
\midrule
\midrule
\multicolumn{4}{c}{FMNIST, $z \sim$ Normal$(0,1)$, $\delta=0.05, m=1$ }\\ %
\midrule
\surt  &86.21$\pm$  0.05& 83.35$\pm$ 0.08& 1 & 1\\
\spgdt &81.91$\pm$ 0.10& 80.28$\pm$ 0.11& 15.74 & 1.97\\
\lzot  &83.83$\pm$ 0.07& 81.79$\pm$ 0.06& 15.49 & 1.88\\
\midrule
\multicolumn{4}{c}{FMNIST, Sigmoid, $\delta=0.05, k \approx 30.63, m=1$ }\\ %
\midrule
\spgdt &81.60$\pm$ 0.11& 80.02$\pm$ 0.08& 12.12 & 1.65\\
\lzot  &83.39$\pm$ 0.10& 81.76$\pm$ 0.10& 12.50 & 1.57\\
\bottomrule
\end{tabular}
\end{sc}
\end{footnotesize}
\end{center}
\vskip -0.3in
\end{table}

\section{Discussions} We propose a novel energy efficient algorithm for direct training of spiking neural networks. We implement the technique for fully connected networks, with two hidden layers. Implementing the method for deeper networks should be straightforward as shown in \cite{nieves2021sparse}. However, we leave the adaptation of the method with convolutional layers for the future work. Our work generates theoretical insights on how zeroth order techniques elegantly handle learning through the discontinuous function such as Heaviside and relates it to the standard surrogate techniques, which we believe can be of interest to future works. 

\section*{Acknowledgement} This work is part of the research project "ENERGY-BASED PROBING FOR SPIKING NEURAL NETWORKS" performed at Mohamed bin Zayed University of Artificial Intelligence~(MBZUAI), funded by  Technology Innovation Institute~(TII) (Contract No. TII/ARRC/2073/2021)

\FloatBarrier
\bibliography{robustness}
\bibliographystyle{icml2023}

\newpage
\appendix
\onecolumn
\section{Proofs of theoretical results}
\lemmaone*
\begin{proof}
Based on our remark above, the only thing left to prove is that the integral $\int_{-\infty}^\infty \E{G^{2}(u; z, \delta)}{z \sim \lambda }du$ is finite. To this end, we have (by using equation \eqref{eq: expected zo})
\begin{align*}
    \int_{-\infty}^\infty \E{G^{2}(u; z, \delta)}{z \sim \lambda }du &= \int_{-\infty}^\infty\frac{1}{\delta}\int_{\frac{|u|}{\delta}}^\infty z^\alpha \lambda(z) dz du = \frac{2}{\delta}\int_{0}^\infty\int_{\frac{|u|}{\delta}}^\infty z^\alpha \lambda(z)dz du\\
    &= \frac{2}{\delta}\int_{0}^\infty\int_{0}^{|z|\delta}|z|^\alpha \lambda(z) du dz = 2\int_{0}^\infty z^{\alpha+1} \lambda(z) dz, 
\end{align*}
which proves the lemma, as by assumptions the resulting integral is finite.
\end{proof}
\thmone*
\begin{proof}
We have \begin{align*}
\frac{d}{du} &\E{c\, h(u+\delta z)}{z \sim \tilde{\lambda} }
= c\,\frac{d}{du}\int_{-\infty}^\infty h(u+\delta z)\Tilde{\lambda}(z)dz=c\, \frac{d}{du}\int_{-\frac{u}{\delta}}^{\infty}\tilde{\lambda}(z) dz
= \frac{c}{\delta}\,\Tilde{\lambda}(-\frac{u}{\delta})=\frac{c}{\delta}\, \Tilde{\lambda}(\frac{u}{\delta}),
\end{align*}
which coincides with \eqref{eq: expected zo}.
\end{proof}
For our next result, which answers the second question that we asked at the beginning of this section, note that a surrogate function is differentiable almost everywhere, which follows from Lebesgue theorem on differentiability of monotone functions. So, taking derivatives here is understood in an ``almost everywhere'' sense.

\thmtwo*
\begin{proof}
Let us assume that $u\geq 0$ (the other case is similar). Then, 
\begin{align*}
\E{cG^2(u;z,\delta)}{z\sim \lambda} = \frac{c}{\delta}\int_{\frac{u}{\delta}}^\infty z^\alpha \lambda(z)dz =-\frac{1}{\delta}\int_{\frac{u}{\delta}}^\infty z^\alpha \frac{\delta^2}{z^\alpha}g'(z\delta)dz
\end{align*}
which after change of variables $u=\delta z$ becomes $g(u)$ and finishes our proof.
\end{proof}
\subsection{Obtaining full-surrogates on Expectation}
\subsubsection{From standard Gaussian}
Recall that the standard normal distribution $N(0,1)$ has PDF of the form $\frac{1}{\sqrt{2\pi}}\exp(-\frac{z^2}{2})$. Consequently, it is straightforward to obtain
\begin{align}
\E{G^{2}(u; z, \delta)}{z \sim \lambda } &=  \frac{1}{\sqrt{2 \pi}}\int_{-\infty}^{\infty} \frac{\abs{z}}{2 \delta} \exp(-\frac{z^2}{2}) dz =\frac{1}{\delta \sqrt{2\pi}} \exp(-\frac{u^2}{2\delta^2}).
\end{align} 

\subsubsection{From Uniform Continuous}
Consider the PDF of a continuous uniform distribution: 
$$f(z; a, b) = \begin{cases}
    \frac{1}{b-a} & \text{for} \, z \in [a,b]\\
    0    & \text{otherwise},
\end{cases}$$ 
where $a<b$ are some real numbers. For the distribution to be even and the resulting scaling constant of the surrogate to be 1 (which translates to $\E{z}{}=0$ and $\E{z^2}{}=1$, respectively) we set, $a=-\sqrt{3}$, $b=\sqrt{3}$. Then,
\begin{align}
\E{G^{2}(u; z, \delta)}{z \sim \lambda }&=  \int_{-\infty}^{\infty}  \frac{\abs{z}}{2\delta} f(z) dz \nonumber \\
    &=\frac{1}{2\sqrt{3}}[\int_{-\sqrt{3}}^{-\frac{\abs{u}}{\delta}}  \frac{\abs{z}}{2 \delta}  dz + \int_{\frac{\abs{u}}{\delta}}^{\sqrt{3}}  \frac{\abs{z}}{2 \delta}  dz]
    =\frac{1}{4\sqrt{3}\delta } z^2 \biggr \rvert_{\frac{\abs{u}}{\delta}}^{\sqrt{3}} \nonumber \\
    &= \begin{cases}
        \frac{1}{4\sqrt{3}\delta } (3 - \frac{u^2}{\delta^2}) & \text{if } \frac{\abs{u}}{\delta} < \sqrt{3}, \\
        0 & \text{otherwise.}
    \end{cases}
 \end{align}

\subsubsection{From Laplacian Distribution}
The PDF of Laplace distribution is given by:
$$f(z; \mu, b) = \frac{1}{2b} \exp(-\frac{\abs{z-\mu}}{b})$$ 
with mean $\mu$ and variance $2b^2$. Setting, $b=\frac{1}{\sqrt{2}}$ and $\mu=0$
and using \eqref{eq:lzo} we obtain,
\begin{align}
    &\E{G^{2}(u; z, \delta)}{z \sim \lambda }
    =\frac{2}{\sqrt{2}}\int_{\frac{\abs{u}}{\delta}}^{\infty}  \frac{\abs{z}}{2 \delta} \exp(-\sqrt{2}\abs{z})dz 
    =\frac{1}{\delta \sqrt{2}}\int_{\frac{\abs{u}}{\delta}}^{\infty}  z \exp(-\sqrt{2}z)dz \nonumber \\
    &=-\frac{1}{\delta \sqrt{2}} (\frac{z}{\sqrt{2}}+\frac{1}{2})  \exp(-\sqrt{2}z)\biggr \rvert_{\frac{\abs{u}}{\delta}}^{\infty} 
    = \frac{1}{2\delta} (\frac{\abs{u}}{\delta}+\frac{1}{\sqrt{2}})  \exp(-\sqrt{2}\frac{\abs{u}}{\delta}).
 \end{align}
\subsection{Expected Back-propagation Thresholds}
To obtain an expected back-propagation threshold, we would like to evaluate:
\begin{align*}
    \tilde{B}_{th}=\delta\E{t}{} = \delta \int_{0}^{\infty} t \hat{\lambda}(t)dt = \delta m \int_{0}^{\infty} t F_{\abs{z}}^{m-1}(t) \tilde{\lambda}(t)dt
\end{align*}
For the standard normal distribution, $\lambda=$ Normal$(0,1)$ we have,
$F_{\abs{z}}(t)=\erf(\frac{t}{\sqrt{2}})$ giving,
\begin{align}
    \tilde{B}_{th}=\frac{2\delta m}{\sqrt{2 \pi}} \int_{0}^{\infty} t \erf^{m-1}(\frac{t}{\sqrt{2}}) \exp(-\frac{t^2}{2})dt.
\end{align}

For uniform continuous, $\lambda=$ Unif$([-\sqrt{3}, \sqrt{3}])$,  we have,  $F_{\abs{z}}(t)=\frac{t}{\sqrt{3}}$ giving,
\begin{align}
    \tilde{B}_{th}=\frac{\delta m}{\sqrt{3}} \int_{0}^{\sqrt{3}} t (\frac{t}{\sqrt{3}})^{m-1}dt = \delta \sqrt{3} \frac{m}{m+1}.
\end{align}

For Laplace distribution, $\lambda=$ Laplace$(0, \frac{1}{\sqrt{2}})$, we have,  $F_{\abs{z}}(t)=1-\exp(-\sqrt{2} t)$,
\begin{align}
    \tilde{B}_{th}=\delta m \sqrt{2} \int_{0}^{\infty} t (1-\exp(-\sqrt{2} t))^{m-1} \exp(-\sqrt{2}t)dt.
\end{align}

\subsection{Simulating a specific Surrogate}
\subsubsection{Sigmoid}
Consider the Sigmoid surrogate function, where the Heaviside is approximated by the differentiable Sigmoid function \cite{zenke2018superspike}. The corresponding surrogate gradient is given by,
\begin{align*}
   \frac{dx}{du} = \frac{d}{du}\frac{1}{1 + \exp(-k u)} = \frac{k \exp(-k u)}{(1 + \exp(-ku))^2} =: g(u)
\end{align*}
and, $$g'(u) = -\frac{k^2 \exp(-ku)(1-\exp(-ku))}{(1+\exp(-ku))^3}$$
Observe that $g(u)$ satisfies our definition of a surrogate ($g(u)$ being even, non-decreasing on $(-\infty,0)$ and 
$\int_{-\infty}^{\infty} g(u) du = 1<\infty$).
Thus, according to Theorem \ref{thm: main2}, we have
\begin{align*}
    &c=-2\delta^2\int_{0}^\infty \frac{g'(t\delta)}{t}dt
    =2\delta^2 k^2\int_{0}^\infty \frac{\exp(-k \delta t)(1-\exp(-k\delta t))}{t(1+\exp(-k\delta t))^3} dt
    = \frac{\delta^2 k^2}{a^2},
\end{align*}
where,  $a:=\sqrt{\frac{1}{0.4262}}$. The corresponding PDF is given by
\begin{align}
    \lambda(z)=-\frac{\delta^2}{c}\frac{g'(\delta t)}{z}
    =a^2\frac{\exp(-k \delta z)(1-\exp(-k\delta z))}{z(1+\exp(-k\delta z))^3}
\end{align}
Observe that the temperature parameter $k$ comes from the surrogate to be simulated, while $\delta$ is used by \lzo. We compute the expected back-propagation threshold of \spgd for $m=1$ as, $\tilde{B}_{th} = \delta \, \E{\abs{z}}{z\sim \lambda }$, with,
\begin{align}
\E{\abs{z}}{z\sim \lambda } &= 2a^2 \int_{0}^{\infty} z  \frac{\exp(-az)(1-\exp(-az))}{z(1+\exp(-az))^3} dz\nonumber =\frac{a}{2} = 0.7659.
\end{align}
\subsubsection{Fast Sigmoid}
Consider also the Fast Sigmoid surrogate gradient \cite{zenke2018superspike, nieves2021sparse} that avoids computing the exponential function in Sigmoid to obtain the gradient: 
\begin{align*}
   \frac{dx}{du} = \frac{1}{(1 + k\abs{u})^2} =: g(u).
\end{align*}
We choose $\alpha=-1$ (note that $\alpha=1$ does not work in this case) and apply theorem \ref{thm: main2} so that,
\begin{align*}
   c&=-2\delta^2\int_{0}^\infty \frac{1}{z^\alpha}g'(z\delta)dz =4\delta^2\int_{0}^\infty \frac{1}{z^\alpha}\frac{k\sign(z \delta )}{(1 + k\abs{z \delta})^3}dz \\
   &=4\delta^2 k \int_{0}^\infty \frac{z}{(1 + k\delta z)^3}dz = \frac{4}{k} \int_{0}^\infty \frac{t}{(1 + t)^3}dt = \frac{2}{k}.
\end{align*}
The PDF is then given by
\begin{align}
\lambda(z)= -\frac{1}{c} \frac{\delta^2}{z^\alpha}g'(z \delta) = k^2\delta^2  \frac{z \sign(z \delta)}{(1 + k\abs{z \delta})^3}. 
\end{align}
To compute the expected back-propagation threshold, we observe,
\begin{align*}
    \tilde{B}_{th}=\delta\E{\abs{z}}{z\sim \lambda } = 2\delta^3 k^2 \int_{0}^{\infty}\frac{z^2 \sign(z \delta)}{(1 + k\abs{z \delta})^3}dz = 2\delta^3 k^2  \int_{0}^{\infty}\frac{ z^2 }{(1 + kz \delta)^3}dz = \frac{2}{k}  \int_{0}^{\infty}\frac{ x^2 }{(1 + x)^3}dx
\end{align*}
The above integral does not converge. However, if we consider finite support [-a, a], we may compute,
$\frac{2}{k}  \int_{0}^{a}\frac{ x^2 }{(1 + x)^3}dx$


\end{document}

%% file: mydef.tex
\usepackage{amsmath, amssymb, amsthm}
\usepackage{thmtools,thm-restate}
\usepackage{algorithm, algorithmic}
\newcommand{\E}[2]{\mathbb{E}_{#2}[#1]}
\newcommand{\m}[1]{\mathbf{#1}}
\newcommand{\norm}[1]{\lVert{#1}\rVert}

\newcommand{\abs}[1]{\lvert{#1}\rvert}
\newcommand{\real}[1]{\mathbb{R}^{#1}}
\newcommand{\R}{\mathbb{R}}

\DeclareMathOperator{\sign}{sign}

\DeclareMathOperator{\erf}{erf}

\theoremstyle{plain}

\theoremstyle{definition}
\newtheorem{defn}{Definition}[section]

\theoremstyle{remark}

\newcommand{\surt}{\textsc{Surr.}\xspace}
\newcommand{\spgdt}{\textsc{Sp.G.}\xspace}
\newcommand{\lzot}{\textsc{L.ZO}\xspace}

\newcommand{\sur}{\textsc{Surrogate}\xspace}
\newcommand{\spgd}{\textsc{SparseGrad}\xspace}
\newcommand{\lzo}{\textsc{LocalZO}\xspace}